\documentclass{article}

\PassOptionsToPackage{numbers, compress}{natbib}
\usepackage[final]{neurips/neurips_2022}

\usepackage{url}            % simple URL typesetting
\usepackage{booktabs}       % professional-quality tables
\usepackage{amsfonts}       % blackboard math symbols
\usepackage[table]{xcolor}         % colors

\usepackage{textcomp}

\usepackage{fancyvrb,moreverb,varwidth}
\VerbatimFootnotes
\usepackage[utf8]{inputenc} % allow utf-8 input
\usepackage[T1]{fontenc}    % use 8-bit T1 fonts
\usepackage{hyperref}       % hyperlinks
\usepackage{subcaption}
\usepackage{amsthm}
\usepackage{amsmath}
\usepackage{amssymb}
\usepackage{booktabs}
\usepackage{pgfplots}
\usepackage{caption} 
\usepackage{standalone} % http://tex.stackexchange.com/questions/11583/insert-a-tikz-picture-from-an-external-file
\usepackage{subfloat}
\usepackage{stmaryrd}

\usepackage{wrapfig}
\newcommand{\bbR}{\mathbb{R}}

\newcommand{\pdiag}{\textnormal{pdiag}}

\newcommand{\calA}{\mathcal{A}}
\newcommand{\calI}{\mathcal{I}}

\newcommand{\E}{\mathbb{E}}

\newcommand{\X}{\mathbf{x}}
\newcommand{\grad}{\nabla}
\newcommand{\C}{\mathcal{C}}
\newcommand{\std}[1]{\tiny{$\pm$ #1}}

\newcommand{\hideh}[1]{}
\newtheorem{lemma}{Lemma}
\newtheorem{theorem}{Theorem}
\newtheorem{proposition}{Proposition}

\title{Efficient Training of \\Low-Curvature Neural Networks}

\author{%
Suraj Srinivas$^{*1}$ \\
Harvard University \\
\texttt{ssrinivas@seas.harvard.edu}
\And
Kyle Matoba$^*$ \\
Idiap Research Institute \& EPFL \\
\texttt{kyle.matoba@epfl.ch}
\And 
Himabindu Lakkaraju \\
Harvard University \\
\texttt{hlakkaraju@hbs.edu}
\And
Fran\c{c}ois Fleuret \\ 
University of Geneva \\
\texttt{francois.fleuret@unige.ch}
}

\begin{document}

\maketitle

\def\thefootnote{*}\footnotetext{Equal Contribution}
\def\thefootnote{1}\footnotetext{Work done partially at Idiap Research Institute}

\begin{abstract}
Standard deep neural networks often have excess non-linearity, making them susceptible to issues such as low adversarial robustness and gradient instability. Common methods to address these downstream issues, such as adversarial training, are expensive and often sacrifice predictive accuracy.

In this work, we address the core issue of excess non-linearity via curvature, and 
demonstrate low-curvature neural networks (LCNNs) that obtain drastically lower curvature 
than standard models while exhibiting similar predictive performance. This leads to improved 
robustness and stable gradients, at a fraction of the cost of standard adversarial training. 
To achieve this, we decompose overall model curvature in terms of curvatures and slopes of 
its constituent layers. To enable efficient curvature minimization of constituent layers, 
we introduce two novel architectural components: first, a non-linearity called centered-softplus that is a stable variant of the softplus non-linearity, and second, a Lipschitz-constrained 
batch normalization layer.

Our experiments show that LCNNs have lower curvature, more stable gradients and increased 
off-the-shelf adversarial robustness when compared to standard neural networks, all without 
affecting predictive performance. Our approach is easy to use and can be readily incorporated 
into existing neural network architectures.

Code to implement our method and replicate our experiments is available at \url{https://github.com/kylematoba/lcnn}. 
\end{abstract}

\section{Introduction}
The high degree of flexibility present in deep neural networks is critical to achieving good performance in complex tasks such as image classification, language modelling and generative modelling of images \cite{he2016deep, chowdhery2022palm, ramesh2022hierarchical}. However, \emph{excessive} flexibility is undesirable as this can lead to model under-specification \cite{d2020underspecification} which results in unpredictable behaviour on out-of-domain inputs, such as vulnerability to adversarial examples. Such under-specification can be avoided in principle via Occam's razor, which requires training models that are as simple as possible for the task at hand, and not any more. Motivated by this principle, in this work we focus on training neural network models without excess non-linearity in their input-output map (e.g.: see Fig \ref{fig:LCNN}), such that predictive performance remains unaffected. 

Central to this work is a precise notion of curvature, which is a mathematical quantity 
that encodes the flexibility or the degree of non-linearity of a function at a point. 
In deep learning, the curvature of a function at a point is often quantified as the norm of the Hessian at that point \cite{Dinh2017, MoosaviDezfooli2018, Dombrowski2022}. Hessian norms are zero everywhere if and only if the function is linear, making them suitable to measure the degree of non-linearity. However, they suffer from a dependence on the scaling of model gradients, which makes them unsuitable to study its interplay with model robustness. In particular, robust models have small gradient norms \cite{hein2017formal}, which naturally imply smaller Hessian norms. But are they truly more linear as a result? To be able to study robustness independent of non-linearity, we propose \textit{normalized curvature}, which normalizes the Hessian norm by its corresponding gradient norm, thus disentangling the two measures. Surprisingly, we find that normalized curvature is a stable measure across train and test samples (see Table \ref{tab:discrepancy}), whereas the usual curvature is not.

A conceptually straightforward approach to train models with low-curvature input-output maps \cite{MoosaviDezfooli2018, qin2019adversarial} involves directly penalizing curvature locally at training samples. However, these methods involve expensive Hessian computations, and only minimize local point-wise curvature and not curvature everywhere. A complementary approach is that of \citet{Dombrowski2019}, who propose architectures that have small global curvature, but do not explicitly penalize this curvature during training. In contrast, we propose efficient mechanisms to explicitly penalize the \textit{normalized curvature} globally. In addition, while previous methods \cite{MoosaviDezfooli2018, Dombrowski2022} penalize the Frobenius norm of the Hessian, we penalize its spectral norm, thus providing tighter and more interpretable robustness bounds.\footnote{Note that the Frobenius norm of a matrix strictly upper bounds its spectral norm}

%For example, given two functions $f,g$ with $\nabla_x f = k \times \nabla_x g$, for some scalar $k > 1$, we have that $\nabla^2_x f = k \times \nabla^2_x g$, which indicates that $f$ is more non-linear than $g$, as $f,g$ are simply scaled versions of each other.
%However, a problem with these definitions is that these do not disentangle the model non-linearity from the steepness of model gradients, and thus fails to properly capture the degree of model non-linearity.

%To illustrate, let us consider two functions $f(x) = 10^{10}x + \sin(x)$ and $g(x) = 10^{-10} x + \sin(x)$, which have equal Hessians. 
%Intuitively, $f(x) \sim 10^{10}x$ is approximately linear while $g(x) \sim \sin(x)$ is more non-linear. However, if measured by the Hessian, both functions are equally non-linear. This is because while the absolute 
%error with the best fit linear model is the same, $f(x) - 10^{10}x = g(x) - 10^{-10}x = \sin(x)$, we 
%care more about the \emph{relative} error which appropriately disentangles `measure of non-linearity' with the magnitude of the linear term. To this end, we propose a modified definition of curvature where the Hessian is scaled by the gradient norm, which we show precisely corresponds to a \emph{relative} local linearity error.

\begin{figure}
  \centering
    \begin{subfigure}{0.27\textwidth}
      \includegraphics[width=\textwidth,clip, trim=2cm 2cm 2cm 2cm]{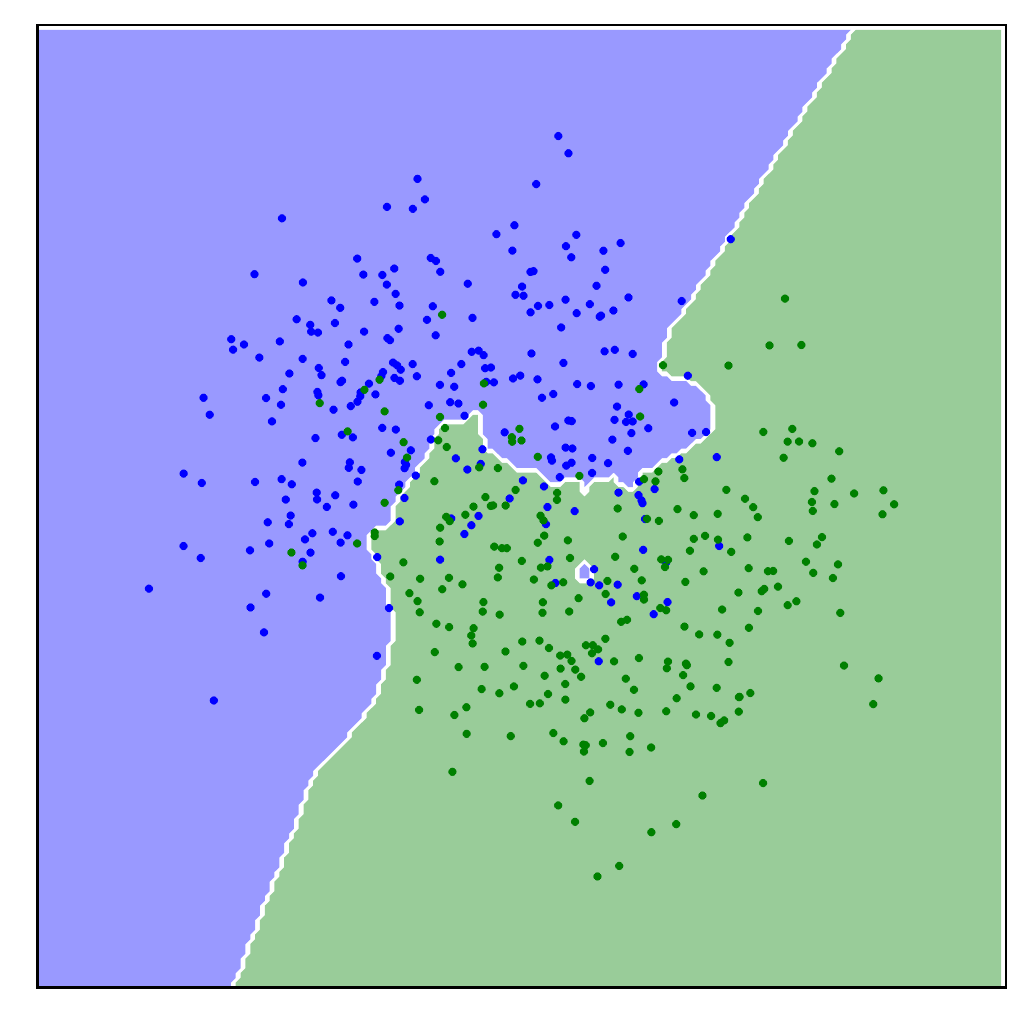}
      \caption{Decision boundary with standard NN ($\beta=1000$)}
      \label{fig:standardNN}
    \end{subfigure}\hspace{3ex}
  \begin{subfigure}{0.27\textwidth} 
      \centering 
       \includegraphics[width=\textwidth,clip, trim=2cm 2cm 2cm 2cm]{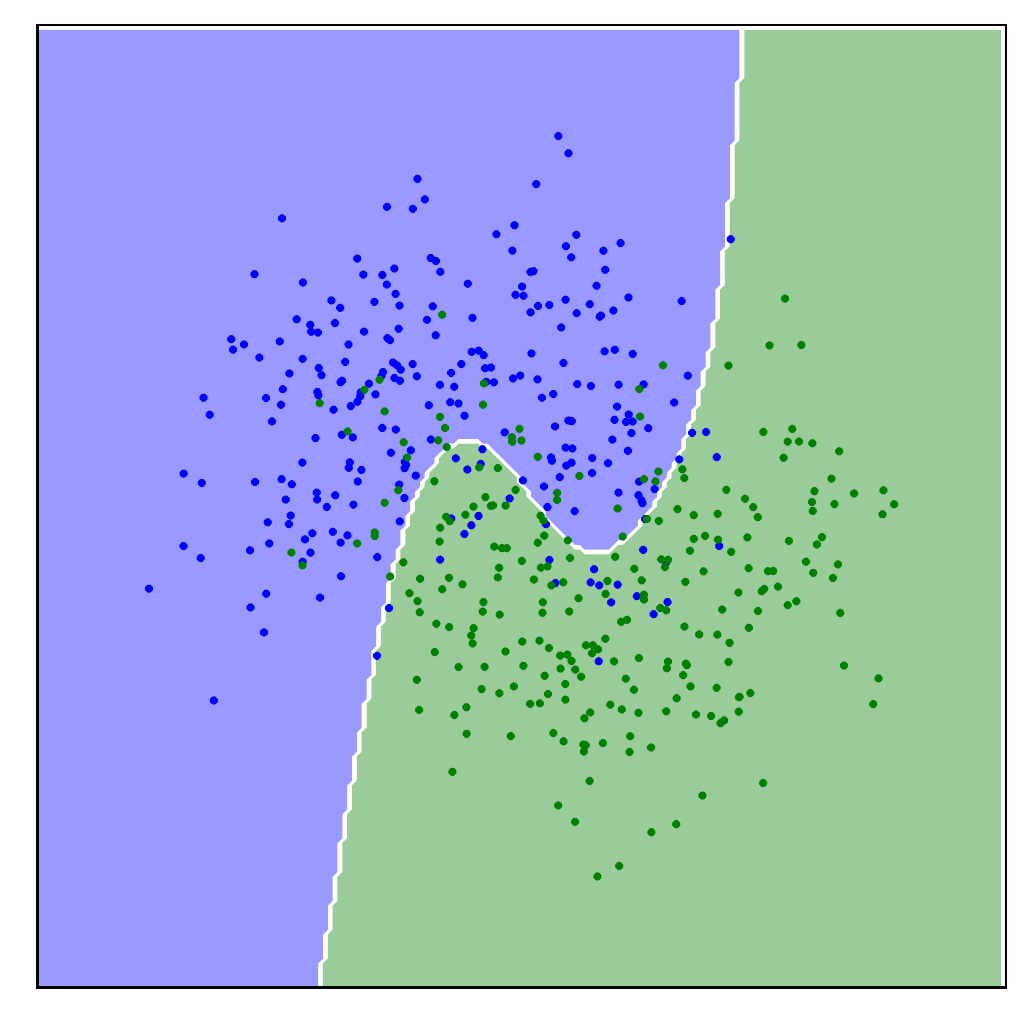}
       \caption{Decision boundary with LCNN ($\beta=1$)}
       \label{fig:LCNN}
     \end{subfigure}\hfill
  \begin{subfigure}{0.35\textwidth}
    \includegraphics[width=0.95\textwidth, clip, trim=2ex 0 1ex 1ex]{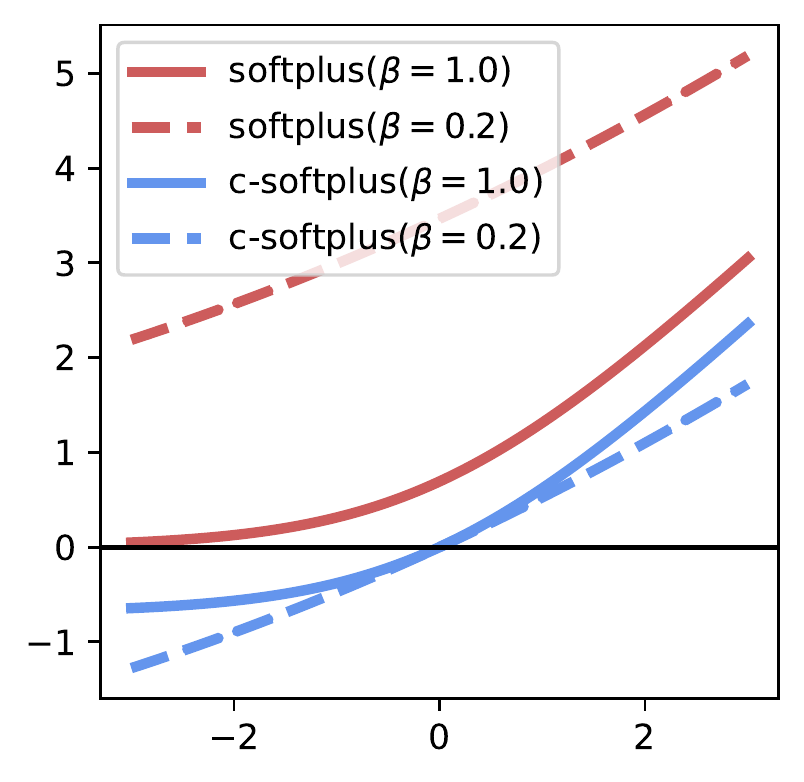}
    \vspace*{-2ex}
    \caption{Softplus vs Centered-Softplus}
    \label{fig:softplus}
  \end{subfigure} 
    \caption{Decision boundaries of (a) standard NN and (b) LCNN trained on the two moons dataset. LCNN recovers highly regular decision boundaries in contrast to the standard NN.
    (c) Comparison of softplus and centered-softplus non-linearities (defined in \S \ref{sec:softplus}). These behave similarly for large $\beta$ values, and converge to linear maps for low $\beta$
    values. However, softplus diverges while centered-softplus stays close to the origin.
    }
\end{figure}

Our overall contributions in this paper are:

\begin{enumerate}
  \item In \S \ref{sec:curvaturedef} we propose to measure curvature of deep models via \emph{normalized curvature}, which is invariant to scaling of the model gradients.
  \item In \S \ref{sec:datafree} we show that normalized curvature of neural networks can be upper bounded in terms of normalized curvatures and slopes of individual layers.
  \item We introduce an architecture for training LCNNs combining a novel activation function called the centered-softplus (\S \ref{sec:softplus}), and recent innovations to constrain the Lipschitz constant of convolutional (\S \ref{sec:convolution}) and batch normalization (\S \ref{sec:batchnorm}) layers.
  \item In \S \ref{sec:properties} we prove bounds on the relative gradient robustness and adversarial robustness
  of models in terms of normalized curvature, showing that controlling normalized curvature directly controls these properties.
  \item In \S \ref{sec:expts}, we show experiments demonstrating that our proposed innovations are successful in training low-curvature models without sacrificing training accuracy, and that such models have robust gradients, and are more robust to adversarial examples out-of-the-box.
\end{enumerate}

\section{Related Work}
\textbf{Adversarial Robustness of Neural Networks:} 
The well-known phenemenon of adversarial vulnerabilities of neural networks 
\citep{Szegedy2014, Goodfellow2014_adv} shows that adding small amounts of 
imperceptible noise can cause deep neural networks to misclassify points with 
high confidence. The canonical method to defend against this vulnerability is 
adversarial training \citep{madry2018towards} which trains models to 
accurately classify adversarial examples generated via an attack 
such as projected gradient descent (PGD). However, this approach is computationally 
expensive and provides no formal guarantees on robustness. \citet{cohen2019certified} 
proposed randomized smoothing, which provides a formal guarantee on robustness 
by generating a smooth classifer from any black-box classifier. 
\citet{hein2017formal} identified the local Lipschitz constant as critical quantity 
to prove formal robustness guarantees. 
% As a result, many works in literature focus on smoothing the input-output map of models to provide adversarial robustness. 
\citet{MoosaviDezfooli2018} penalize the Frobenius norm of the Hessian, and show that they 
performs similarly to models trained via adversarial training.
\citet{qin2019adversarial} introduce a local linearity regularizer, which also 
implicitly penalizes the Hessian. Similar to these works, we enforce low curvature to induce robustness, but we 
focus on out of the box robustness of LCNNs.
%Finally, \citet{zhang2019theoretically} introduced TRADES,
%a simple approach where models are explicitly regularized to be locally robust via a robustness loss instead of penalizing the Hessian or gradients. 

\textbf{Unreliable Gradient Interpretations in Neural Networks:} Gradient explanations in neural networks
can be unreliable. \citet{Ghorbani2019, zhang2020interpretable} showed 
that for any input, it is possible to find adversarial inputs such that the gradient explanations for 
these points that are highly dissimilar to each other. \citet{srinivas2020rethinking} showed that pre-softmax logit gradients however are independent of model 
behaviour, and as a result we focus on post-softmax loss gradients in this work. 
\citet{Ros2018} showed empirically that robustness can be improved by gradient regularization, however \citet{Dombrowski2019} showed that gradient instability is primarily due to large Hessian norms. 
This suggests that the gradient penalization in \citet{Ros2018} performed 
unintentional Hessian regularization, which is consistent with our experimental results.
To alleviate this, \citet{Dombrowski2022} proposed to train low curvature models via softplus
activations and weight decay, broadly similar to our strategy. 
However, while \citet{Dombrowski2022} focused on the Frobenius norm of the Hessian, 
we penalize the normalized curvature -- a scaled version of the Hessian spectral norm -- 
which is strictly smaller than the Frobenius norm, and this results in a more sophisticated penalization strategy. 

\textbf{Lipschitz Layers in Neural Networks} There has been extensive work on methods to 
bound the Lipschitz constant of neural networks. \citet{Cisse2017} introduced 
Parseval networks, which penalizes the deviation of linear layers from orthonormality -- 
since an orthonormal linear operator evidently has a Lipschitz constant of one, 
this shrinks the Lipschitz constant of a layer towards one. \citet{Trockman2021} use a 
reparameterization of the weight matrix, called the Cayley Transform, that is orthogonal 
by construction. \citet{Miyato2018, Ryu2019} proposed 
spectral normalization, where linear layers are re-parameterized by dividing by their 
spectral norm, ensuring that the overall spectral norm of the parameterized layer is one. 

\section{Training Low-Curvature Neural Networks}
In this section, we introduce our approach for training low-curvature neural nets
(LCNNs). Unless otherwise specified, we shall consider a neural network classifier $f$ that maps inputs
$\X \in \bbR^d$ to logits which characterize the prediction, and can be further combined with the true label distribution and a loss function to give a scalar loss value $f(\X) \in \bbR^+$.   

\subsection{Measuring Relative Model Non-Linearity via Normalized Curvature}
\label{sec:curvaturedef}
We begin our analysis by discussing a desirable definition of curvature
$\C_f(\X) \in \bbR^+$. While curvature is well-studied topic in differential geometry \cite{lee2006riemannian} where Hessian normalization is a common theme, our discussion will be largely independent of this literature, 
in a manner more suited to requirements in deep learning. Regardless of the definition, 
a typical property of a curvature measure is that 
$\C_f(\X) = 0 ~~\forall~ \X \in \bbR^d \iff f$ is linear, and the higher the curvature, the further from linear the function is. Hence $\max_{\X \in \bbR^d} \C_f(\X)$ can be thought of as a measure of a model's non-linearity. A common way to define curvature in
machine learning \cite{Dinh2017, MoosaviDezfooli2018, Dombrowski2022} has been via Hessian norms. 
However, these measures are sensitive to gradient scaling, which is undesirable. After all, the degree of model non-linearity must intuitively be independent of how large its gradients are.

For example, if two functions $f,g$ are scaled ($f = k \times g$), or rotated 
versions of each other ($\grad f = k \times \grad g$), then we would like them to have similar 
curvatures in the interest of disentangling curvature (i.e., degree of non-linearity) from scaling. 
It is easy to see that Hessian norms $ \| \nabla^2_{\X} f(\X) \|_2$ do not have this property, as scaling the function also scales the corresponding Hessian. We would like to be able to recover low-curvature models with
steep gradients (that are non-robust), precisely to be able to disentangle their properties from low-curvature models
with small gradients (that are robust), which Hessian norms do not allow us to do. 

To avoid this problem, we propose a definition of curvature that is approximately \textit{normalized}: $\C_f(\X) = \| \grad^2 f(\X) \|_2 / (\| \grad f(\X) \|_2 + \varepsilon)$. Here $\| \grad f(\X) \|_2$ and $\| \grad^2 f(\X) \|_2$ are the $\ell_2$ norm of the gradient and the spectral norm of the Hessian, respectively, where $ \grad f(\X) \in \bbR^d, \grad^2 f(\X) \in \bbR^{d \times d}$, and $\varepsilon > 0$ is a small constant to ensure well-behavedness of the measure. This definition measures Hessian norm \textit{relative} to the gradient norm, and captures a notion of relative local linearity, which can be seen via an application of Taylor's theorem:

\begin{equation}
  \underbrace{\frac{\| f(\X + \epsilon) - f(\X) - \grad f(\X)^\top \epsilon \|_2}{ \| \grad f(\X) \|_2 } }_{\text{relative local linearity}} \leq \frac{1}{2} \underbrace{\max_{\X \in \bbR^d} \C_f(\X)}_{\text{normalized curvature}} \| \epsilon \|_2. 
\end{equation}

Here the numerator on the left captures the local linearity error, scaled by the gradient norm in the denominator,
which can be shown using Taylor's theorem to be upper bounded by the normalized curvature. We shall consider penalizing this notion of normalized curvature, and we shall simply refer to this quantity as curvature in the rest of the paper.

\subsection{A Data-Free Upper Bound on Curvature}
\label{sec:datafree}

Directly penalizing the curvature is computationally expensive as
it requires backpropagating an estimate of the Hessian norm, which 
itself requires backpropagating gradient-vector products. This requires 
chaining the entire computational graph of the model at least three times. 
\citet{MoosaviDezfooli2018} reduce the complexity of this operation by computing a finite-difference approximation to the Hessian from gradients, but even this double-backpropagation is expensive. We require an efficient penalization procedure takes a single backpropagation step.

To this end, we propose to minimize a data-free upper bound on the curvature. To illustrate the idea, we first show this upper bound for the simplified case of the composition of one-dimensional functions ($f: \bbR \rightarrow \bbR$).

\begin{lemma}
Given a 1-dimensional compositional function $f = f_L \circ f_{L-1}\circ ...\circ f_1$ with 
$f_i: \bbR \rightarrow \bbR$ for $i = 1, 2, \hdots, L$, the normalized curvature 
$\C_f := |f'' / f'|$, is bounded by
$\left|\frac{f''}{f'}\right| \leq \sum_{i=1}^{L} \left|\frac{f_i''}{f_i'}\right| ~ \prod_{j=1}^{i} |f'_j|$
% where $f', f''$ are the first and second order derivatives respectively.
\end{lemma}
\vspace*{-1ex}
\textit{Proof.}
We first have $f' = \prod_{i=1}^{L} f'_i$. Differentiating this expression we have 
$f'' = \sum_{i=1}^L f_i'' \prod_{j=1}^i f_j' \prod_{k=1, k \neq i}^L f'_k$. 
Dividing by $f'$, taking absolute value of both sides, and using the triangle 
inequality, we have the intended result. 

Extending the above expression to functions $\bbR^n \rightarrow \bbR$ is not
straightforward, as the intermediate layer Hessians are order three tensors. We derive this 
using a result from \citet{Wang2017} that connects the spectral norms of order-$n$ tensors
to the spectral norm of their matrix ``unfoldings''. The full derivation is presented in the appendix, 
and the (simplified) result is stated below.

\begin{theorem}
\label{thm:main}
Given a function $f = f_L \circ f_{L-1} \circ \hdots \circ f_1$ with 
$f_i: \bbR^{n_{i-1}} \rightarrow \bbR^{n_i}$, the curvature $\C_f$ can be bounded by 
the sum of curvatures of individual layers $\C_{f_i}(\X)$, i.e.,
\begin{equation}
\label{eqn:curvature}
  \C_f(\X) \leq \sum_{i=1}^{L}n_i \times \C_{f_i}(\X) ~ \prod_{j=1}^{i} \| \nabla_{f_{j-1}    } f_j(\X) \|_2 
         \leq \sum_{i=1}^{L} n_i \times \max_{\X'} \C_{f_i}(\X') ~ \prod_{j=1}^{i} \max_{\X'} \| \nabla_{f_{j-1}} f_j(\X') \|_2. 
 %\C_f(\X) \leq \sum_{i=1}^{L} \C_{f_i}(\X) ~ \prod_{j=1}^{i} \| \nabla_{f_{j-1}(\X)} f_j(\X) \|_2 
 %       \leq \sum_{i=1}^{L}            \max_{\X'} \C_{f_i}(\X') ~ \prod_{j=1}^{i} \max_{\X'} \| :w
 %\nabla_{f_{j-1}(\X')} f_j(\X') \|_2
\end{equation}
\end{theorem}

% \begin{theorem}
% \label{thm:main}
% Given a function $f = f_L \circ f_{L-1} \circ \hdots \circ f_1$ with 
% $f_i: \bbR^{n_{i-1}} \rightarrow \bbR^{n_i}$, the curvature $\C_f$ can be bounded by 
% the sum of curvatures of individual layers $\C_{f_i}(\X)$, i.e.,
% \begin{equation}
% \label{eqn:curvature}
% \C_f(\X) \leq \sum_{i=1}^{L} \C_{f_i}(\X) ~ \prod_{j=1}^{i} \| \nabla_{f_{j-1}(\X)} f_j(\X) \|_2 \leq \sum_{i=1}^{L} \max_\X \C_{f_i}(\X) ~ \prod_{j=1}^{i} \max_{\X} \| \nabla_{f_{j-1}(\X)} f_j(\X) \|_2.
% \end{equation}
% \end{theorem}

The rightmost term is independent of $\X$ and thus holds uniformly across all data points. 
This bound shows that controlling the curvature and Lipschitz constant of each layer of a neural
network enables us to control the overall curvature of the model. 

Practical neural networks typically consist of\footnote{We ignore self-attention layers in this work.} linear maps such as convolutions, fully connected layers, and batch normalization layers, and non-linear activation functions. Linear maps have zero curvature by definition, and non-linear layers often have bounded
gradients ($=1$), which simplifies computations. In the sections that follow, we shall see how to penalize the remaining terms, i.e., the curvature of
the non-linear activations, and the Lipschitz constant of the linear layers. 

\subsection{Centered-Softplus: Activation Function with Trainable Curvature}
\label{sec:softplus}

Theorem \ref{thm:main} shows that the curvature of a neural network depends 
on the curvature of its constituent activation functions. 
We thus propose to use activation functions with minimal curvature.

The kink at the origin of the ReLU function implies an undefined second derivative. 
On the other hand, a smooth activation such as the 
softplus function, $s(x; \beta) = \log(1 + \exp(\beta x)) / \beta$ 
is better suited to analyzing questions of curvature. 
%This function was introduced (with $\beta = 1$) 
%as the activation of a deep neural network by \cite{Dugas2000}, 
%who found that the resemblance of the derivatives to the ``greeks'' of options pricing 
%models gave performance superior to a sigmoidal activation. 
Despite not being a common baseline choice, softplus does see regular use, especially where its smoothness 
facilitates analysis \cite{Grathwohl2018}. The curvature of the softplus function is

\begin{equation}
\C_{s(\cdot; \beta)}(x) = \beta \times \left(1 - \frac{\mathrm{d}s(x; \beta)}{\mathrm{d}x}\right) \leq \beta
\end{equation}

Thus using softplus with small $\beta$ values ensures low curvature.
However, we observe two critical drawbacks of softplus preventing its usage 
with small $\beta$: 
(1) divergence for small $\beta$, where $s(x; \beta \rightarrow 0) = \infty$ which ensures that well-behaved low curvature maps cannot be recovered and,
(2) instability upon composition i.e., 
$s^n(x=0 ; \beta) = \underbrace{s \circ s \circ ... \circ s}_{n~\text{times}}(x = 0; \beta) = \frac{\log(n + 1)}{\beta}$.
Thus we have that $s^n(x=0; \beta) \rightarrow \infty$ as $n \rightarrow \infty$, which shows that composing softplus 
functions exacerbates instability around the origin. This is critical for deep models with large number of layers $n$, which is 
precisely the scenario of interest to us. To
remedy this problem, we propose \textit{centered-softplus} $s_0(x; \beta)$, a simple modification to softplus by introducing a normalizing term as follows.  

\begin{align}
\label{eqn:def_s0}
s_0(x;\beta) = s(x;\beta) - \frac{\log 2}{\beta} = \frac{1}{\beta} \log \left( \frac{1 + \exp(\beta x)}{2} \right). 
\end{align}

This ensures that $s(x=0; \beta) = s^n(x=0 ; \beta) = 0$ for any positive integer $n$, and hence also ensures stability upon composition. More importantly, we have $s_0(x; \beta \rightarrow 0) = x / 2$ which is a scaled linear map, while still retaining $s_0(x; \beta \rightarrow \infty) =
\text{ReLU}(x)$. This ensures that we are able to learn both well-behaved linear maps, as well as highly
non-linear ReLU-like maps if required.

We further propose to cast $\beta$ is a learnable parameter and penalize its value, hence directly penalizing the curvature of that layer. Having accounted for the curvature of the non-linearities, the next section discusses controlling the gradients of the linear layers.

\subsection{Lipschitz Linear Layers}

Theorem \ref{thm:main} shows that penalizing the Lipschitz constants of the 
constituent linear layers in a model is necessary to penalize overall model curvature. 
% This is identical to the spectral norm of the transpose of the linear layer, which is its maximal
% eigenvalue, i.e., $\| A \|_2 = \lambda_{\max}(A)$ when $A$ is a matrix. 
There are broadly three classes of linear layers we consider: convolutions, 
fully connected layers, and batch normalization. % We shall consider these cases separately in the following discussion. 

\subsubsection{Spectrally Normalized Convolutions and Fully Connected Layers}
\label{sec:convolution}
We propose to penalize the Lipschitz constant of convolutions and fully connected layer via existing spectral normalization-like techniques. For fully connected layers, we use vanilla spectral normalization (\citet{Miyato2018}) which  controls the spectral norm of a fully connected layer by reparameterization -- replacing a weight matrix 
$W$, by $W / ||W||_2$ which has a unit spectral norm. 
 \citet{Ryu2019} generalize this to convolutions by presenting a power iteration method that works directly on the linear mapping implicit in a convolution: maintaining 3D left and right singular ``vectors'' 
of the 4D tensor of convolutional filters and developing the corresponding update rules. They call this 
``real'' spectral normalization to distinguish it from the approximation that \cite{Miyato2018} proposed 
for convolutions. Using spectral normalization on fully connected layers and ``real'' spectral normalization 
on convolutional layers ensures that the spectral norm of these 
layers is exactly equal to one, further simplifying the bound in Theorem \ref{thm:main}.

\subsubsection{$\gamma$-Lipschitz Batch Normalization}
\label{sec:batchnorm}
In principle, at inference time batch normalization (BN) is multiplication by the diagonal matrix of
inverse running standard deviation estimates.\footnote{We ignore the learnable parameters of batch normalization for simplicity. The architecture we propose subsequently also does not have trainable affine parameters.}.
Thus we can spectrally normalize the BN layer by computing the reciprocal of the smallest
running standard deviation value across dimensions $i$, i.e., $||\text{BN}||_2 = \max_{x} ||\text{BN}(x)||_2 
= 1 / \min_i \text{running-std}(i)$. In practice, we found that models
with such spectrally normalized BN layers tended to either diverge or fail to train
in the first place, indicating that the scale introduced
by BN is necessary for training. To remedy this, we introduce the
$\gamma$-\textit{Lipschitz} batch normalization, defined by

\begin{align*}
\text{1-Lipschitz-BN}(\X) &\leftarrow \text{BN}(\X) / || \text{BN} ||_2\\
\gamma \text{-Lipschitz-BN}(\X) &\leftarrow \underbrace{\min (\gamma, || \text{BN}||_2)}_{\text{scaling factor} \leq \gamma} \times {\text{1-Lipschitz-BN}(\X)}.
\end{align*}

By clipping the scaling above at $\gamma$ (equivalently, the running standard deviation below, at $1 / \gamma$), we can ensure that the Lipschitz constant of a batch normalization layer is at most equal to $\gamma \in \bbR^+$. We provide a PyTorch-style code snippet in the appendix. % Thus the Lipschitz constant of
% the modified layer is at most $\gamma$ at test time. 
As with $\beta$, described in \S \ref{sec:softplus}, we cast $\gamma$ as a learnable parameter in order to penalize it during training. \citet{gouk2021regularisation} proposed a similar (simpler) solution, whereas they they fit a common $\gamma$ to all BN layers,  we let $\gamma$ vary freely by layer. 

\subsection{Penalizing Curvature}
We have discussed three architectural innovations -- centered-softplus activations with a 
trainable $\beta$, spectrally normalized linear and convolution layers and a $\gamma$-Lipschitz 
batch normalization layer. We now discuss methods to penalize the overall curvature of a model built with these layers. 

% We first note that our curvature bound is simple for models without batch-normalization, and depends 
% only on the curvature of the centered-softplus denoted here by $s_0$, and is given by $\C_f \leq \sum_{\substack{i=1\\
%                             f_i \in s_0}}
%                             ^L \beta_i $, which this directly constitutes the regularization penalty for such models.

Since convolutional and fully connected layers have spectral norms equal to one by construction, 
they contribute nothing to curvature. Thus, we restrict attention to batch normalization and activation layers. The set of which will be subsequently referred to respectively $\beta\textnormal{SP}$ and $\gamma\textnormal{BN}$. For models with batch normalization, naively using the upper bound 
in Theorem \ref{thm:main} is problematic due to the exponential growth in the product of Lipschitz constants of batch normalization layers. To alleviate this, we propose to use 
a penalization $\mathcal{R}_f$ where $\gamma_i$ terms are aggregated additively across batch normalization
layers, and independent of the $\beta_i$ terms in the following manner: 

\begin{align}
\mathcal{R}_f = \lambda_{\beta} \times \sum_{i \in \beta\textnormal{SP}} \beta_i + 
      \lambda_{\gamma} \times \sum_{j\in \gamma\text{BN}} \log \gamma_j 
\end{align}

An additive aggregation ensures that the penalization is well-behaved during training and does not 
grow exponentially. Note that the underlying model is necessarily linear if the penalization term is zero,
thus making it an appropriate measure of model non-linearity. Also note that $\beta_i \geq 0, \gamma_i \geq 1$ by construction. We shall henceforth use
the term ``LCNN'' to refer to a model trained with the proposed architectural components (centered-softplus, spectral
normalization, and $\gamma$-Lipschitz batch normalization) and associated regularization terms on
$\beta, \gamma$. We next discuss the robustness and interpretability benefits that we obtain with LCNNs.

\section{Why Train Low-Curvature Models?}\label{sec:properties}
In this section we discuss the advantages that low-curvature models offer, particularly as it pertains to robustness and gradient stability. These statements apply not just to LCNNs, but low-curvature models in general obtained via any other mechanism.

\subsection{Low Curvature Models have Stable Gradients}\label{sec:gradrobustness}
Recent work \citep{Ghorbani2019, zhang2020interpretable} has shown that gradient explanations are manipulable, and that we can easily find inputs whose explanations differ maximally from 
those at the original inputs, making them unreliable in practice for identifying important features. As we shall show, 
this amounts to models having a large curvature. In particular, we show that the relative gradient difference 
across points $\| \epsilon \|_2$ away in the $\ell_2$ sense is upper bounded by the normalized curvature $\C_f$,
as given below.

\begin{proposition}
  \label{prop:gradrobust}
  Consider a model $f$ with $\max_{\X} \C_f(\X) \leq \delta_{\C}$,
  and two inputs $\X$ and $\X + \epsilon$ ($\in \bbR^d$). The relative distance between gradients at 
  these points is bounded by

\begin{align*}
  \dfrac{\| \grad f(\X + \epsilon) - \grad f(\X) \|_2}{\| \grad f(\X) \|_2 } 
    \leq \| \epsilon \|_2 \delta_{\C} \exp( \| \epsilon \|_2 \delta_{\C})
    \sim {\| \epsilon \|_2 \C_f(\X)} ~~~~\text{(Quadratic Approximation)}
\end{align*}
\end{proposition}
% Taylor's theorem, and bounding of the residual error obtained by the mean-value theorem, which is also given in the appendix.

The proof expands $f(\X)$ in a Taylor expansion around $\X + \epsilon$ and bounds the magnitude of the second and higher order terms over the neighborhood of $\X$, and the full argument is given in the appendix. The upper bound is considerably simpler when we assume that the function locally quadratic, which corresponds to the 
rightmost term $\| \epsilon \|_2 \C_f(\X)$. Thus the smaller the model curvature, the more locally stable are the gradients.

\subsection{Low Curvature is Necessary for $\ell_2$ Robustness}\label{sec:advrobustness}

Having a small gradient norm is known to be an important aspect of 
adversarial robustness \citep{hein2017formal}. However, 
small gradients alone are not sufficient, and low curvature is 
necessary, to achieve robustness. This is easy to see intuitively - a model may have 
low gradients at a point leading to robustness for small noise values, but if the 
curvature is large, then gradient norms at neighboring points can quickly increase, 
leading to misclassification for even slightly larger noise levels. 
This effect is an instance of \textit{gradient-masking} \citep{athalye2018obfuscated} , which provides an illusion of robustness by making models only locally robust.

In the result below, we formalize this intuition and establish an 
upper bound on the distance between two nearby points, which we show depends on 
both the gradient norm (as was known previously) and well as the max curvature of the 
underlying model.

\begin{proposition}
  Consider a model $f$ with
   $\max_{\X}\C_f(\X) \leq \delta_{\C}$, then
  for two inputs $\X$ and $\X + \epsilon$ ($\in \bbR^d$), we have the 
  following expression for robustness

  \begin{align*}
    \| f(\X + \epsilon) - f(\X) \|_2 
      &\leq \| \epsilon \|_2 \| \grad f(\X) \|_2 \left( 1 + 
          \| \epsilon \|_2 \delta_{\C} \exp( \| \epsilon \|_2 \delta_{\C})  \right)\\
      &\sim \| \epsilon \|_2  \| \grad f(\X) \|_2 \left( 1 + 
      \| \epsilon \|_2 \C_f(\X) \right) ~~~~ (\text{Quadratic Approximation})
  \end{align*}
\end{proposition}

The proof uses similar techniques to those of proposition \ref{prop:gradrobust}, 
and is also given in the appendix. This result shows that given two models with 
equally small gradients at data points, the greater robustness will be achieved by the 
model with the smaller curvature. 

\section{Experiments}
\label{sec:expts}

In this section we perform experiments to (1) evaluate the effectiveness of our proposed method in
training models with low curvature as originally intended, (2) evaluate whether low curvature models
have robust gradients in practice, and (3) evaluate the effectiveness of low-curvature models for
adversarial robustness. Our experiments are primarily conducted on a base ResNet-18 architecture (\cite{He2016})
using the CIFAR10 and CIFAR100 datasets (\cite{Krizhevsky2009}), and using the Pytorch \cite{Paszke2017} framework. Our methods entailed fairly modest computation -- our most involved computations can be completed in under three GPU days, and all experimental results could be computed in less than 60 GPU-days. We used a mixture of GPUs -- primarily NVIDIA GeForce GTX 1080 Tis -- on an internal compute cluster.

\textbf{Baselines and Our Methods} Our baseline model for comparison involves using a ResNet-18 
model with softplus activation with a high $\beta=10^3$ to mimic ReLU, and yet have well-defined curvatures. 
Another baseline
is gradient norm regularization (which we henceforth call `GradReg') \cite{Drucker1992}, where the
same baseline model is trained with an additional penalty on the gradient norm. We train two variants of our approach - a base LCNN, which involves penalizing the curvature, and another variant combining LCNN penalty and gradient norm regularization (LCNN + GradReg), which controls both the 
curvature and gradient norm. Our theory indicates that the LCNN + GradReg variant is likely to produce 
more robust models, which we verify experimentally. %As in the native PyTorch implementation of softplus, when $x \beta > 20$ we use the identity map for numerical stability. 
We also compare with CURE \cite{MoosaviDezfooli2018}, softplus with weight decay \cite{Dombrowski2022} and adversarial training with $\ell_2$ PGD \cite{madry2018towards} with noise magnitude of $0.1$ and 3 iterations of PGD. We provide experimental details in the appendix.

\textbf{Parameter Settings} All our models are trained for 200 epochs
with an SGD + momentum optimizer, with a momentum of $0.9$ and an initial learning rate of $0.1$
which decays by a factor of 10 at 150 and 175 epochs, and a weight decay of $5 \times 10^{-4}$.

\subsection{Evaluating the Efficacy of Curvature Penalization}
In this section, we evaluate whether LCNNs indeed reduce model curvature in practice. 
Table \ref{tab:geometry} contains our results, from which we make the following observations:
(1) most of our baselines except CURE and adversarial training do not meaningfully
lose predictive performance
(2) GradReg and adversarially trained models are best at reducing gradient norm while LCNN-based models are best at penalizing curvature. Overall, these experimental results show that LCNN-based models indeed minimize curvature as intended.

%\suraj{add comment about loss Hessian contains the "Gauss-Newton" term of gradient outer products, which explains why gradreg minimizes Hessian as well.}

We also observe in Table \ref{tab:geometry} that GradReg \cite{Drucker1992} has an unexpected regularizing effect on the Hessian and curvature. We conjecture that this is partly due to the following decomposition of the loss Hessian, which can be written as $\grad^2 f(\X) \sim \grad f_{l}(\X) \nabla^2_{f_l} \text{LSE}(\X) \grad f_{l}(\X)^{\top} + \nabla_{f_l} \text{LSE}(\X) \grad^2 f_l(\X)$, where $\text{LSE}(\X)$ is the LogSumExp
function, and $f_l(\X)$ is the pre-softmax logit output. We observe that the first term strongly 
depends on the gradients, which may explain the Hessian regularization effects of GradReg, while the second term depends on the Hessian of the bulk of the neural network,
which is penalized by the LCNN penalties. This also explains why combining both penalizations (LCNN + GradReg) further reduces curvature.

We also measure the average per-epoch training times on a GTX 1080Ti, which are: standard models / softplus + weight decay ($\sim 100$ sec), LCNN ($\sim 160$ sec), GradReg ($\sim 270$ sec), LCNN+GradReg ($\sim 350$ sec), CURE / Adversarial Training ($\sim 500$ sec). 
Note that the increase in computation for LCNN is primarily due to the use of spectral normalization layers. The results show that LCNNs are indeed able to penalize curvature by only marginally ($1.6 \times$) increasing training time, and using LCNN+GradReg only increases time $1.3 \times$ over GradReg while providing curvature benefits.

\begin{table}[h]
  \begin{center}
    \caption{Model geometry of various ResNet-18 models trained with various
    regularizers on the CIFAR100 test dataset. Gradient norm
    regularized models \cite{Drucker1992} (`GradReg') are best at reducing gradient norms, 
    while LCNN-based models are best at reducing curvature, leaving gradients unpenalized. 
    We obtain the benefits of both by combining these penalties. Results are averaged across two runs.}
  \label{tab:geometry}
    \begin{tabular}{c | c | c | c | c} 
     \textbf{Model} & $\E_{\X} \| \grad f(\X) \|_2$ & $\E_{\X} \| \grad^2 f(\X) \|_2$ & $\E_{\X} \C_f(\X)$ & Accuracy (\%) \\ 
     \midrule

     Standard & 19.66 \std{0.33} & 6061.96 \std{968.05} & 270.89 \std{75.04} & 77.42 \std{0.11}   \\ 
     LCNNs & 22.04 \std{1.41} & 1143.62 \std{99.38} & 69.50 \std{2.41} & 77.30 \std{0.11}\\
     GradReg \cite{Drucker1992} & \textbf{8.86} \std{0.12} & 776.56 \std{63.62} & 89.47 \std{5.86} & 77.20 \std{0.26}  \\ 
     LCNNs + GradReg & 9.87 \std{0.27} & \textbf{154.36} \std{0.22} & \textbf{25.30} \std{0.09} & 77.29 \std{0.07}  \\
     \midrule
     %GradReg (strong) & \textbf{5.31} & 208.86 & 40.19 & 75.08  \\ 
     %LCNNs (strong) & 13.2 & 265.84 & 33.39 & 73.63 \std{0.87} \\
     %LCNNs + GradReg (strong) & 5.17 & \textbf{38.73} & \textbf{12.49} & 75.33 \std{0.01}  \\
     %\midrule
     CURE \cite{MoosaviDezfooli2018} & \textbf{8.86} \std{0.01} & 979.45 \std{14.05} & 116.31 \std{4.58} & 76.48 \std{0.07}\\
     Softplus + Wt. Decay \cite{Dombrowski2022} & 18.08 \std{0.05} & 1052.84 \std{7.27} & 70.39 \std{0.88} & 77.44 \std{0.28}  \\
     Adversarial Training \cite{Madry2017} & \textbf{7.99} \std{0.03} & 501.43 \std{18.64} & 63.79 \std{1.65} & 76.96 \std{0.26}\\ 
     %Adversarial Training ($\| \epsilon \|_2 = 0.3$) & 5.06 & 156.25 & 30.75 & 75.21\\ 
    \end{tabular}

  \end{center}
  \end{table}

\subsection{Impact of Curvature on Gradient Robustness}

In \S \ref{sec:gradrobustness}, we showed that low-curvature models tend to have more robust gradients. 
Here we evaluate whether this prediction empirically by measuring the relative gradient robustness for the models 
with various ranges of curvature values and noise levels. In particular, we measure robustness to random 
noise at fixed magnitudes ranging logarithmically from $1 \times 10^{-3}$ to $1 \times 10^{-1}$. We plot our results in Figure \ref{fig:gradrobustness}
where we find that our results match theory (i.e, the simplified quadratic approximation in \S \ref{sec:gradrobustness}) quite closely in terms of the overall trends, and that low curvature models have an \textbf{order of magnitude} improvement in robustness over standard models. % The close correspondence with theory also shows that deep neural networks can be well-approximated locally by a quadratic.

\begin{figure}
     \centering
    \begin{subfigure}[b]{0.47\textwidth}  
    \includegraphics[clip, trim=0.4cm 0.4cm 0.3cm 0.3cm, width=\textwidth]{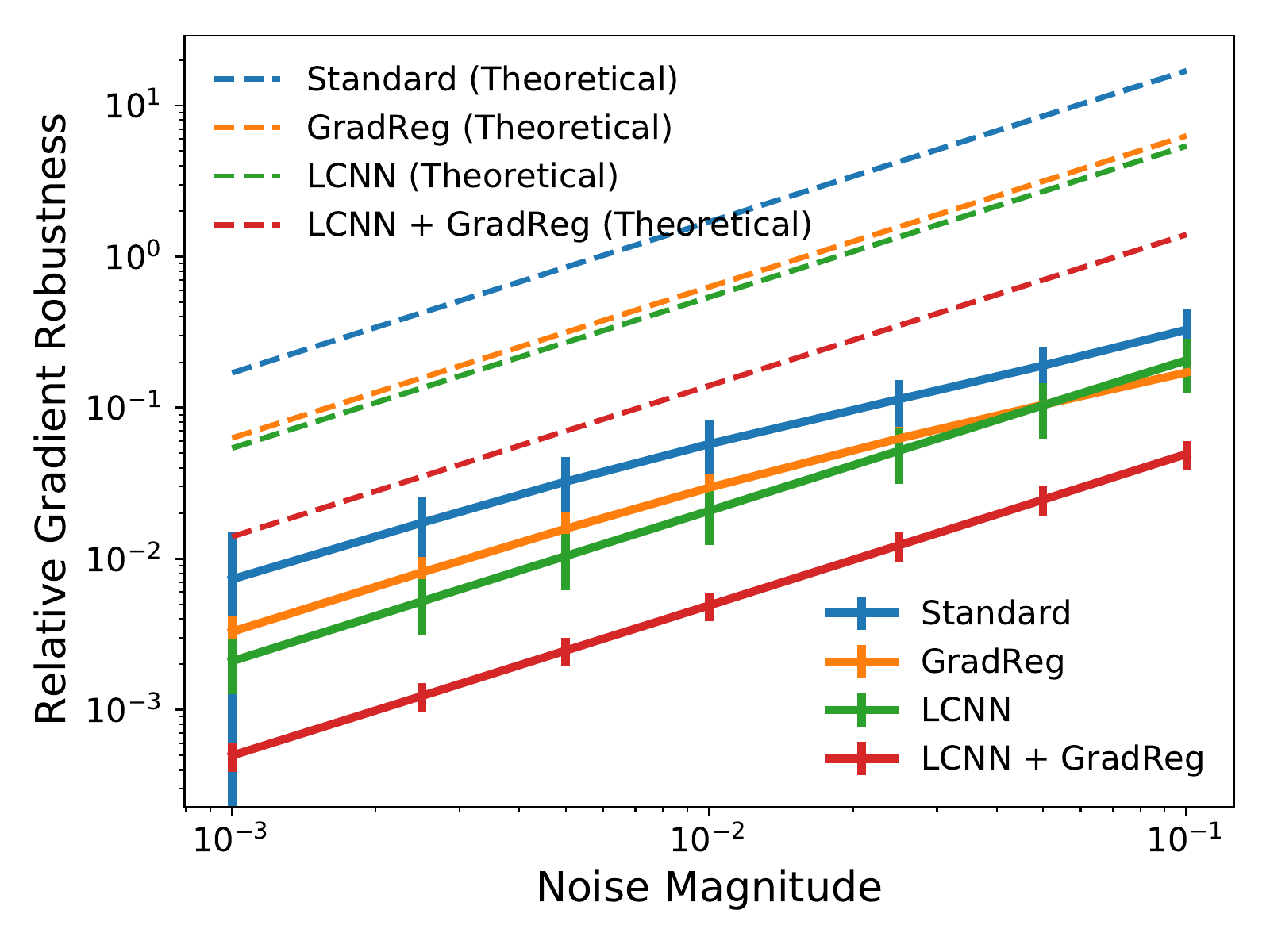}
      \caption{Gradient Robustness on CIFAR10}
      \label{subfig:a}
    \end{subfigure}\hfill 
    \begin{subfigure}[b]{0.47\textwidth}  
    \includegraphics[clip, trim=0.4cm 0.4cm 0.3cm 0.3cm, width=\textwidth]{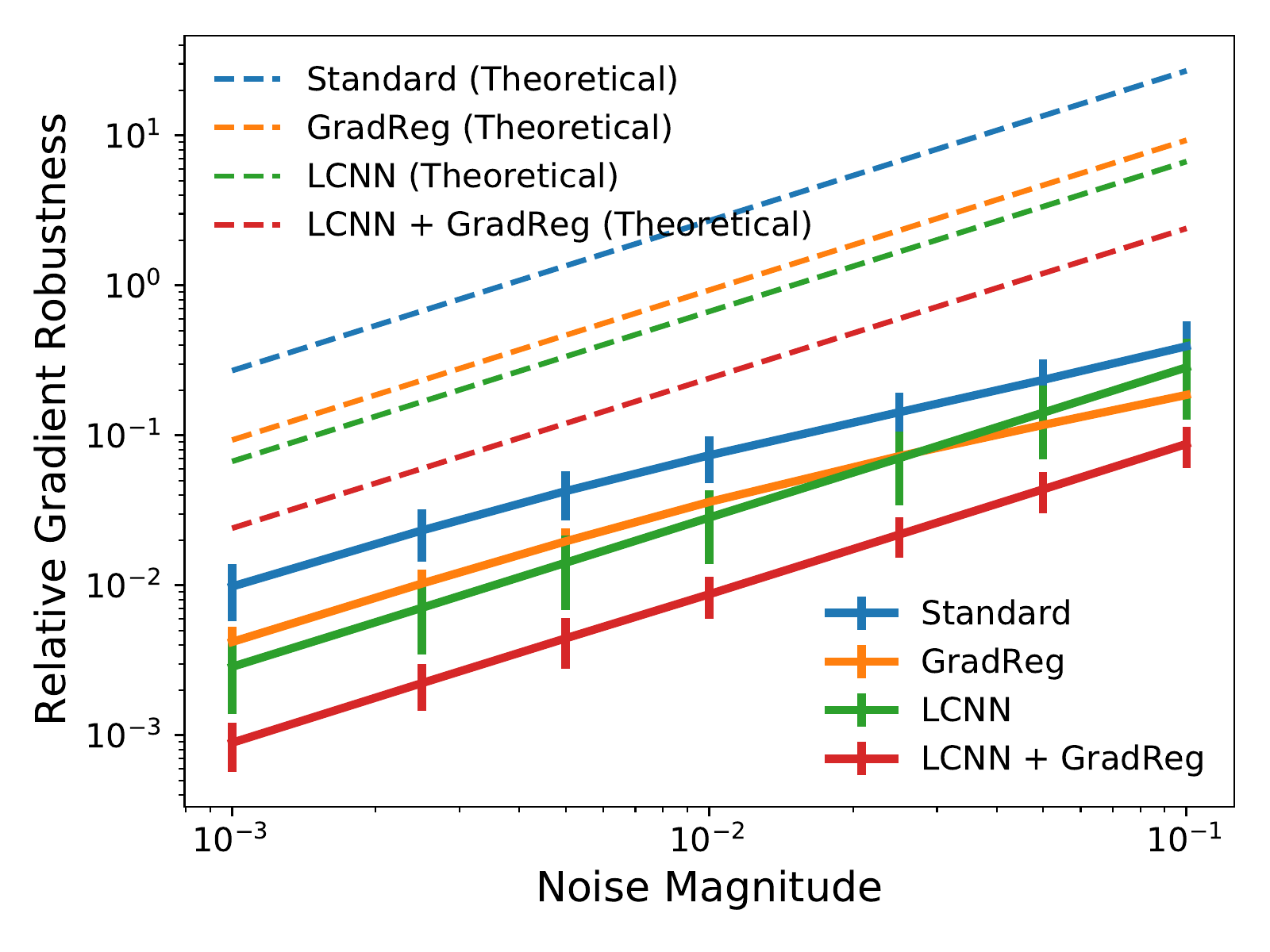}
      \caption{Gradient Robustness on CIFAR100}
      \label{subfig:b}
    \end{subfigure}

        \caption{Plot showing relative gradient robustness $\frac{\| \grad f(\X + \epsilon) - \grad f(\X) \|_2}{\| \grad f(\X) \|_2}$ as a function of added noise $\| \epsilon \|_2$ on (a) CIFAR10 and (b) CIFAR100
        with a ResNet-18 model. We observe that low-curvature models lead to an \textbf{order of magnitude} improvement 
        in gradient robustness, and this improvement closely follows the trend predicted by the theoretical upper bound in \S \ref{sec:properties}.}
        \label{fig:gradrobustness}
\end{figure}

\subsection{Impact of Curvature on Adversarial Robustness}
Our theory in \S \ref{sec:advrobustness} shows that having low curvature is necessary for robustness, along with having 
small gradient norms. In this section we evaluate this claim empirically, by evaluating 
adversarial examples via $\ell_2$ PGD \citep{madry2018towards} adversaries with various noise magnitudes. 
We use the Cleverhans library \cite{Papernot2018} to implement PGD. Our results are present in Table \ref{tab:advex} where we find that LCNN+GradReg models perform on par with adversarial training, without a resulting accuracy loss.

\begin{table}[h]
  \begin{center}
    \caption{Results indicating off-the-shelf model accuracies (\%) upon using $\ell_2$ PGD adversarial examples across various noise magnitudes. Adversarial training performs the best overall, however sacrifices clean accuracy. LCNN+GradReg
    models perform similarly but without significant loss of clean accuracy. Results are averaged across two runs.
    } 
  \label{tab:advex}
    \begin{tabular}{c | c | c | c | c | c} 
     \textbf{Model} & Acc. (\%) & $\| \epsilon \|_2=0.05$ & $\| \epsilon \|_2=0.1$ & $\| \epsilon \|_2=0.15$ & $\| \epsilon \|_2=0.2$ \\ 
     \midrule
     Standard & 77.42 \std{.10} & 59.97 \std{.11} & 37.55 \std{.13} & 23.41 \std{.08} & 16.11 \std{.21}   \\ 
     LCNN & 77.16 \std{.07} & 61.17 \std{.53} & 39.72 \std{.17} & 25.60 \std{.32} & 17.66 \std{.18} \\
     GradReg &  77.20 \std{.26} & 71.90 \std{.11} & 61.06 \std{.03} & 49.19 \std{.12} & 38.09 \std{.47}  \\ 
     LCNNs + GradReg & 77.29 \std{.26} & \textbf{72.68} \std{.52} & \textbf{63.36} \std{.39} & \textbf{52.96} \std{.76} & \textbf{42.70} \std{.77}  \\  
     \midrule
     CURE \cite{MoosaviDezfooli2018} & 76.48 \std{.07} & 71.39 \std{.12} & 61.28 \std{.32} & 49.60 \std{.09} & 39.04 \std{.16}\\
     Softplus + Wt. Decay \cite{Dombrowski2022} & 77.44 \std{.28} & 60.86 \std{.36} & 38.04 \std{.43} & 23.85 \std{.33} & 16.20 \std{.01} \\
     %\midrule 
     Adversarial Training \cite{Madry2017} & 76.96 \std{.26} & \textbf{72.76} \std{.15} & \textbf{64.70} \std{.20} & \textbf{54.80} \std{.25} & \textbf{44.98} \std{.57}\\
    \end{tabular}

  \end{center}
  \end{table}

\subsection{Train-Test Discrepancy in Model Geometry}
During our experiments, we observed a consistent phenomenon where the 
gradient norms and Hessian norms for test data points were much larger than those for 
the train data points, which hints at a form of overfitting with regards to these quantities. 
We term this phenomenon as the \textit{train-test discrepancy} in model geometry. Interestingly, 
we did not observe any such discrepancy using our proposed curvature measure, indicating that 
our proposed measure may be a more reliable measure of model geometry. We report our results in 
Table \ref{tab:discrepancy}, where we measure the relative discrepancy, finding that the discrepancy for our 
proposed measure of curvature is multiple orders of magnitude smaller than the corresponding quantity 
for gradient and Hessian norms. We leave further investigation of this phenomenon -- regarding why curvature is stable across train and test -- as a topic for future work. 

\begin{table*}[h]
  \begin{center}
    \caption{Train-test descrepancy in model geometry, where the 
    relative descrepancy $\Delta_{tt} g(\mathcal{X}) = |\frac{g(\mathcal{X}_{\textnormal{test}}) - g(\mathcal{X}_{\textnormal{train}})}{g(\mathcal{X}_{\textnormal{test}})}|$ 
    is shown for three different geometric measures. We observe that 
    (1) there exists a large train-test descrepancy, with the test
    gradient / hessian norms being $ > 10\times $ the corresponding values for the train set.
    (2) the descrepancy is 2-3 orders of magnitude smaller for our proposed curvature measure, 
    indicating that it may be a stable model property.} 
    \label{tab:discrepancy}
    \begin{tabular}{c | c | c | c } 
     \textbf{Model} & $ \Delta_{tt} \E_{\X \in \mathcal{X}} \| \grad f(\X) \|_2$ & $ \Delta_{tt} \E_{\X \in \mathcal{X}} \| \grad^2 f(\X) \|_2$ & $ \Delta_{tt} \E_{\X \in \mathcal{X}} \C_f(\X)$\\ 
     \midrule

     Standard & 11.75 & 12.28 & \textbf{0.025}   \\ 
     GradReg &  11.33 & 11.22 & \textbf{0.017}  \\ 
     LCNN & 19.99 & 11.33 & \textbf{0.129} \\
     LCNNs + GradReg & 21.82 & 10.43 & \textbf{0.146}  \\  
    
    \end{tabular}

  \end{center}
  \end{table*}

\paragraph{Summary of Experimental Results} Overall, our experiments show that:

(1) LCNNs have lower curvature than standard models as advertised, and combining them with gradient norm regularization further decreases curvature (see Table \ref{tab:geometry}). The latter phenomenon is unexpected, as our curvature measure ignores gradient scaling. 

(2) LCNNs combined with gradient norm regularization achieve an order of magnitude improved gradient robustness over standard models (see Figure \ref{fig:gradrobustness}).

(3) LCNNs combined with gradient norm regularization outperform adversarial training in terms of achieving a better predictive accuracy at a lower curvature (see Table \ref{tab:geometry}), and are competitive in terms of adversarial robustness (see Table \ref{tab:advex}), while being $\sim 1.4\times$ faster. 

(4) We observe that there exists a train-test discrepancy for standard geometric quantities like the gradient and Hessian norm, and this discrepancy disappears for our proposed curvature measure (see Table \ref{tab:discrepancy}). 

We also present ablation experiments, additional adversarial attacks, and evaluations on more datasets and architectures in the Appendix.

\section{Discussion}
In this paper, we presented a modular approach to remove excess curvature in neural network models. Importantly, we found that combining vanilla LCNNs with gradient norm regularization resulted in models with the smallest curvature, the most stable gradients as well as those that are the most adversarially robust. Notably, this procedure achieves adversarial robustness \textbf{without} explicitly generating adversarial examples during training. 

The current limitations of our approach are that we only consider convolutional and fully connected layers, and not self-attention or recurrent layers. We also do not investigate the learning-theoretic benefits (or harms) of low-curvature models, or study their generalization for small number of training samples, or their robustness to label noise (which we already observe in Fig. \ref{fig:LCNN}). Investigating these are important topics for future work.

\begin{ack}
The authors would like to thank the anonymous reviewers for their helpful feedback and all the funding agencies listed below for supporting this work. SS and HL are supported in part by the NSF awards $\#$IIS-2008461 and $\#$IIS-2040989, and research awards from Google, JP Morgan, Amazon, Harvard Data Science Initiative, and D$^3$ Institute at Harvard. KM and SS (partly) are supported by the Swiss National Science Foundation under grant number FNS-188758 ``CORTI''. HL would like to thank Sujatha and Mohan Lakkaraju for their continued support and encouragement.
\end{ack}

\bibliographystyle{unsrtnat}
\bibliography{neurips/biblio}

\appendix

\section{Proof of Theorem 1}

In this section, we provide the proof of the main theorem in the paper, which decomposes overall 
curvature into curvatures and slopes of constituent layers. We state Theorem 1 below for reference.

\begin{theorem}

Given a function $f = f_L \circ f_{L-1} \circ \hdots \circ f_1$ with $f_i: \bbR^{n_{i-1}} \rightarrow \bbR^{n_i}$, the curvature $\C_f$ can be bounded as follows 

\begin{equation}
\C_f(\X) \leq \sum_{i=1}^{L} n_i \times \C_{f_i}(\X) ~ \prod_{j=1}^{i} \| \nabla_{f_{j-1}} f_j(\X) \|_2 \leq \sum_{i=1}^{L} n_i \times \max_{\X'} \C_{f_i}(\X') ~ \prod_{j=1}^{i} \max_{\X'} \| \nabla_{f_{j-1}} f_j(\X') \|_2. \label{eqn:curvature}
\end{equation}
\end{theorem}

This statement is slightly different than the one given in the paper, differing by a term of the width of each nonlinear layer. Since we do not use or care about the units of curvature, only its minimization, and have elected to equally-weight each term of the sum, this is an inconsequential discrepancy.

A similar bound is constructed recursively in \cite{Singla2020}. \cite{Dombrowski2022} gives a similar formula, albeit for the Frobenius norm. The Frobenius norm is both simpler, because the sum of squared entries is independent of the layout of the data, and also weaker, since it cannot deliver a bound which holds uniformly in the data. To our knowledge \autoref{eqn:curvature} is the first explicit, easily-interpreted formula of its type. 

We start with some preliminaries, with the actual proof being in \autoref{sec:putting_together}.

% To proceed with the proof, we first mention some preliminaries regarding tensor calculus.

\subsection{Derivatives of compositional functions}

For a function $f: \bbR^d \rightarrow \bbR^r$, let $\nabla f: \bbR^d \rightarrow \bbR^{d \times r}$
 denote its gradient, and $\nabla^2 f: \bbR^{d} \rightarrow \bbR^{d  \times r \times d}$ denote 
its Hessian. We drop the argument to functions when possible, and all norms will be spectral 
norms. 

Given $L$ functions $f_i: \bbR^{n_{i-1}} \rightarrow \bbR^{n_i}, i = 1 \hdots, L$ let $f_{k, k + j} = f_{k+j} \circ f_{k + j-1} \circ \hdots \circ f_k: \bbR^{n_{k-1}} \rightarrow \bbR^{n_{k + j}}$ for $1 \le k \le k + j \le L$. This function composition will be our model of a deep neural network where $f_\ell$ represents the action of the $\ell$th layer. % Let $F$ be shorthand for $F_{1,L}$, the behavior of the entire network. 

If each $f_i$ is continuously differentiable, we have this formula for the gradient of $f_{k, k + j}$ 
\begin{align}
\label{eqn:network_gradient}
\nabla f_{k, k + j} &= \prod_{i=1}^{j} \nabla f_{k + j - i + 1} \in \bbR^{n_{k + j} \times n_k}.
\end{align}

where we adopt the convention that $f_{j,j}(x) = x$ in order to make the boundary conditions correct. The product begins at the end, with $\nabla f_{k + j}$ and progresses forward through the indices -- a straightforward consequence of the chain rule of differentiation. Supposing moreover that each $f_i$ is twice-differentiable, we have this formula for the second derivative of $f_{1, k}$:

\begin{equation}
\begin{aligned}
\label{eqn:network_hessian}
\nabla^2 f_{1, k} &= \sum_{i=1}^{k} \nabla^2 f_{1, i} - \nabla^2 f_{1, i-1} \\
\textnormal{ where } (\nabla^2 f_{1, i} - \nabla^2 f_{1, i-1}) &= (\nabla^2 f_i)\left(\nabla f_{1, i-1}, \nabla f_{i+1, k}^\top, \nabla f_{1, i-1}\right) \in \bbR^{n_0 \times n_k \times n_0}
\end{aligned}
\end{equation}

where we have used the \emph{covariant multilinear matrix multiplication} notation: $\nabla^2 f_i$ is an order-three tensor $\in \bbR^{n_i \times n_i \times n_i}$, with the first and third modes multiplied by $\nabla f_{1, i-1} \in \bbR^{n_i \times n_0}$ and the second mode multiplied by $\nabla f_{i+1, k}^\top \in \bbR^{n_i \times n_k}$.

\subsection{Tensor calculus}
\label{eqn:tensor_calculus}

In this section, we present a simplifed version of the notation from \cite{Wang2017}. A $k$-linear map is a function of $k$ (possibly multidimensional) variables such that if any $k - 1$ variables are held constant, the map is linear in the remaining variable. A $k$-linear function is can be represented by an order-$k$ tensor $\calA$ given elementwise by $\calA = \llbracket a_{j_1 \hdots j_k} \rrbracket \in \bbR^{d_1 \times \hdots \times d_k}$. 

The covariant multilinear matrix multiplication of a tensor with matrices (order 2 tensors) $M_1 = (m_{i_1j_1}^{(1)}) \in \bbR^{d_1 \times s_1}, \hdots, M_k = (m_{i_kj_k}^{(k)}) \in \bbR^{d_k \times s_k}$ is

\begin{align*}
\calA(M_1, \hdots, M_k) = \left\llbracket \sum_{i_1 = 1}^{d_1} \hdots \sum_{i_k}^{d_k} a_{i_1\hdots i_k} m_{i_1j_1}^{(1)} \hdots m_{i_k j_k}^{(k)} \right\rrbracket \in \bbR^{s_1 \times \hdots \times s_k}.
% \llbracket a \rrbracket
\end{align*}

This operation can be implemented via iterated einsums as:
\begin{verbatim}
def covariant_multilinear_matmul(a: torch.Tensor,
                                 mlist: List[torch.Tensor]) -> torch.Tensor:
    order = a.ndim
    base_indices = string.ascii_letters
    indices = base_indices[:order]
    next_index = base_indices[order]

    val = a
    for idx in range(order):
        resp_str = indices[:idx] + next_index + indices[idx+1:]
        einsum_str = indices + f",{indices[idx]}{next_index}->{resp_str}"
        val = torch.einsum(einsum_str, val, mlist[idx])
    return val
\end{verbatim}

For example, covariant multilinear matrix multiplication of an order two tensor is pre- and post- multiplication by its arguments: $M_1^\top \calA M_2 = \calA(M_1, M_2)$. The generalization of the matrix spectral norm is $|| \calA ||_2 = \sup \{ \calA(x_1, \hdots , x_k) : ||x_i || = 1, x_i \in \bbR^{d_i}, i = 1, 2, \hdots, k\}$. The computation of order-$k$ operator norms is hard in theory, and also in practice (cf. \cite{Friedland2018}). In order to address this difficulty, we introduce an instance of the \emph{unfold} operator. 

% For $\calA \in \bbR^{d_1 \times d_2 \times d_3}$, $\textnormal{unfold}_{\{\{1\}, \{2, 3\}\}}(\calA) \in \bbR^{d_1 \times d_2d_3}$ is the matrix with the $j$th row being the flattened (in C, not Fortran, order) $j$th $d_2 \times d_3$ matrix.\footnote{In PyTorch notation, $\textnormal{unfold}_{\{\{1\}, \{2, 3\}\}}(\texttt{a}) = \texttt{a.reshape(a.shape[0], -1)}$} Unfolding is useful because it it allows us to bound an order-3 operator norm in terms of order-2 operator norms -- 
For $\calA \in \bbR^{d_1 \times d_2 \times d_3}$, $\textnormal{unfold}_{\{\{1, 2\}, \{3\}\}}(\calA) \in \bbR^{d_1d_2 \times d_3}$ is the matrix with the $j$th column being the flattened $j$th (in the final index) $d_2 \times d_3$ matrix.\footnote{In PyTorch notation, $\textnormal{unfold}_{\{\{1, 2\}, \{3\}\}}(\texttt{a}) = \texttt{torch.flatten(a, end\_dim=1)}$} Unfolding is useful because it it allows us to bound an order-3 operator norm in terms of order-2 operator norms -- 
\citet{Wang2017} shows that $||\calA|| \le ||\textnormal{unfold}_{\{\{1, 2\}, \{3\}\}}(\calA)||$. The upper bound -- the operator norm of a \emph{matrix} -- can computed with standard largest singular-value routines. A similar unfolding-based bound was used in the deep-learning context by \cite{Singla2019} to give improved estimates on the spectral norm of convolution operators.

% To facilitate the analysis of unfolded tensors, we coin operations that \emph{put to} and \emph{get from} the diagonal of tensors: 
To facilitate the analysis of unfolded tensors, we coin operations that \emph{put to} the diagonal of tensors: 

\begin{itemize}
\item $\pdiag2: \bbR^d \mapsto \bbR^{d \times d}$ defined by $\textnormal{pdiag}2(x)_{ij} = 
\begin{cases} 
x_j &\textnormal{ if } i = j \\
0 &\textnormal{ otherwise.}
\end{cases}$
\item $\textnormal{pdiag}3: \bbR^d \mapsto \bbR^{d \times d \times d}$ defined by $\textnormal{pdiag}3(x)_{ijk} = 
\begin{cases} 
x_j &\textnormal{ if } i = j = k \\
0 &\textnormal{ otherwise.} 
\end{cases}$
% \item $\textnormal{gdiag}2: \bbR^{d \times d} \rightarrow \bbR^d$ defined by $\textnormal{gdiag}2(x)_j = x_{jj}$.
% \item $\textnormal{gdiag}3: \bbR^{d \times d \times d} \rightarrow \bbR^d$ defined by $\textnormal{gdiag}3(x)_j = x_{jjj}$.
\end{itemize}

Further, let $1_n \in \bbR^n$ be a vector of ones, $I_n = \pdiag2(1_n) \in \bbR^{n \times n}$ be the $n$-dimensional identity matrix, and $\calI_n = \pdiag3(1_n) \in \bbR^{n \times n \times n}$. For two vectors $a \in \bbR^n, b \in \bbR^n$, let $ab$ denote the elementwise product. $\otimes$ denotes the well-understood Kronecker product, so that, for example, $1_n^\top \otimes I_m$ is an $m \times nm$ matrix consisting of $n$ copies of the $m \times m$ identity matrix stacked side by side. Where it is redundant, we drop the subscripts indicating dimension.

We use the following facts about tensors, their unfoldings, and their operator norms, in what follows. 

\begin{enumerate}
\item $\calA = \pdiag3(ab) \implies \calA(M_1, M_2, M_3) = \calI(\pdiag2(a)M_1, M_2, \pdiag2(b)M_3)$
\item $\textnormal{unfold}_{\{\{1, 2\}, \{3\}\}}(\calI(M_1, M_2, M_3)) = (M_1 \otimes I_{s_2})^\top \textnormal{unfold}_{\{\{1, 2\}, \{3\}\}}(\calI(I_{d_1}, M_2, M_3))$
% \item $(1_{d_1}^\top \otimes I_{s_2}) \textnormal{unfold}_{\{\{1, 2\}, \{3\}\}}(\calI(I_{d_1}, M_2, M_3)) = M_2^\top M_3$
\item $||(1_{d_1}^\top \otimes I_{s_2}) \textnormal{unfold}_{\{\{1, 2\}, \{3\}\}}(\calI(I_{d_1}, M_2, M_3))|| \le ||M_2^\top M_3|| s_2$
% \item Let $M \in \bbR^{r \times c}$. For all $p \in \bbN$, $P = (1_p^\top \otimes I_r) M \implies ||M|| \le ||P||$
% \item $|| ||$
%    Fact #1: unfold_{[0, 1], [2]} I(M1, M2, M3) = torch.kron(M1, I).T @ (unfold_{[0, 1], [2]} I(I, M2, M3))
%    Fact #2: collapser @ (unfold_{[0, 1], [2]} I(I, M2, M3)) = M2.T @ M3
%    Fact #3: collapser @ M = P implies ||M|| <= ||P||
%    Fact #4: If cala = pdiag3(ab) then
%                cala(M1, M2, M3) = I(pdiag2(a) @ M1, M2, pdiag2(b) @ M3)
\end{enumerate}

Taken together Fact \#2 and \#3 imply that 

\begin{align}
\label{eqn:essential}
||\textnormal{unfold}_{\{\{1, 2\}, \{3\}\}}(\calI(M_1, M_2, M_3))|| \le ||M_1 || \times ||M_2^\top M_3|| \times s_2
\end{align}
which is our essential bound for the norm of an order-3 tensor in terms of order-2 tensors. 
% Briefly: Fact\#1, \#2, and \#3 follow simply by manipulating indices. Fact \#4 follows because if $||M|| = $ 

% Taken together Fact \#3 and \#4 imply that 
% \begin{align}
% \label{eqn:main_lemma}
% ||\textnormal{unfold}_{\{\{1, 2\}, \{3\}\}}(\calI(I_{d_1}, M_2, M_3))|| \le ||M_2^\top M_3||,
% \end{align}

\subsection{Hessian increment bound}
Let $\sigma(x) = \exp(x) / (1 + \exp(x)))$ denote the (elementwise) logistic function $\bbR^d \mapsto \bbR^d$. The derivatives of $s(x; \beta)$ can be written as 

\begin{align}
\label{eqn:softplus_first_derivative}
\nabla s(x; \beta) &= \textnormal{pdiag}2(\sigma(\beta x)) \in \bbR^{d \times d} \\
\label{eqn:softplus_second_derivative}
\nabla^2 s(x; \beta) &= \pdiag3(\beta \sigma(\beta x)(1 - \sigma(\beta x))) \in \bbR^{d \times d \times d}.
\end{align}

Let the $i$th softplus layer have coefficient $\beta_i$, then the increment from \autoref{eqn:network_hessian}, can be bounded as follows:
% \begin{equation}
% \begin{aligned}
\begin{align}
& \label{eqn:incr_bound0}
||(\nabla^2 f_i)\left(\nabla f_{1, i-1}, \nabla f_{i+1, k}^\top, \nabla f_{1, i-1}\right)||  \\
\label{eqn:incr_bound1}
=& ||\beta_i \calI_{n_i}\left(\pdiag2(1 - \sigma(\beta_i x)) \nabla f_{1, i-1}, \nabla f_{i+1, k}^\top, \pdiag2(\sigma(\beta_i x))\nabla f_{1, i-1}\right)|| \\
\label{eqn:incr_bound2}
\le& || \beta_i \textnormal{unfold}_{\{\{1, 2\}, \{3\}\}}(\calI_{n_i}(\pdiag2(1 - \sigma(\beta_i x)) \nabla f_{1, i-1}, \nabla f_{i+1, k}^\top, \pdiag2(\sigma(\beta_i x))\nabla f_{1, i-1})|| \\
\label{eqn:incr_bound3}
\le& ||\beta_i (\pdiag2(1 - \sigma(\beta_i x)) \nabla f_{1, i-1} || \times ||\nabla f_{i+1, k}^\top \pdiag2(\sigma(\beta_i x))\nabla f_{1, i-1}|| \times n_i \\
\label{eqn:incr_bound4}
\le& n_i \times ||\beta_i (\pdiag2(1 - \sigma(\beta_i x)) || \times ||\nabla f_{1, i-1}|| \times ||\nabla f_{i+1, k}^\top \nabla f_i  \nabla f_{1, i-1}|| \\
\label{eqn:incr_bound5}
=& n_i \times \C_{f_i} \times ||\nabla f_{1, i-1}|| \times ||\nabla f||.
\end{align}
% \end{aligned}
% \end{equation}
\autoref{eqn:incr_bound1} follows by Fact \#1 above, along with \autoref{eqn:softplus_second_derivative}. 
\autoref{eqn:incr_bound2} is the standard unfolding bound by \cite{Wang2017}. 
\autoref{eqn:incr_bound3} is our main bound on order-3 tensors in terms of order-2 matrices, \autoref{eqn:essential}. 
\autoref{eqn:incr_bound4} follows from the Cauchy-Schwartz inequality. The replacement in the last term of the product is \autoref{eqn:softplus_first_derivative}. 
% \footnote{It is known (\cite{Laub2004} via \cite{Woodrow2017}) that the operator norm of a Kroenecker product is the product of the operator norms: $||b^\top \otimes a|| = ||b^\top|| \times ||a||$.}
\autoref{eqn:incr_bound5} rewrites \autoref{eqn:incr_bound4} using \autoref{eqn:network_gradient}: $ f_{i+1, k}^\top \nabla f_i \nabla f_{1, i-1} = \nabla f_{1, k}$ and the definition of the curvature of $f_i$. 
 
\subsection{Putting it together}
\label{sec:putting_together}
Ignoring the $\epsilon$ term, 
\begin{align}
\label{eqn:curvature_bound}
\C_f(\X) &= \frac{\| \grad^2 f(\X) \| }{\| \grad f(\X) \|} \\
\label{eqn:put_together0}
         &= \frac{1}{\| \grad f(\X) \|} \left| \left| \sum_{i=1}^L \nabla^2 f_{1, i}(\X) - \nabla^2 f_{1, i-1}(\X) \right| \right| \\
\label{eqn:put_together1}
         &\le \frac{1}{\| \grad f(\X) \|} \sum_{i=1}^L ||\nabla^2 f_{1, i}(\X) - \nabla^2 f_{1, i-1}(\X)|| \\
\label{eqn:put_together2}
         &\le \sum_{i=1}^L n_i \times \C_{f_i}(\X) \times ||\nabla f_{1, i-1}(\X)|| \\
\label{eqn:put_together3}
         &\le \sum_{i=1}^L n_i \times \C_{f_i}(\X) \times \prod_{i=1}^j ||\nabla f_{i}(\X)|| \\
\label{eqn:put_together4}
         &\le \sum_{i=1}^L n_i \times \max_{\X} \C_{f_i}(\X) \times \prod_{i=1}^j \max_{\X} ||\nabla f_{i}(\X)||
% \sum_{i=1}^{L} \C_{f_i}(\X) ~ \prod_{j=1}^{i} \| \nabla_{f_{j-1}} f_j(\X) \|_2 \leq \sum_{i=1}^{L} \max_\X \C_{f_i}(\X) ~ \prod_{j=1}^{i} \max_{\X} \| \nabla_{f_{j-1}} f_j(\X) \|_2.
\end{align}
\autoref{eqn:put_together0} substitutes \autoref{eqn:network_hessian}.
\autoref{eqn:put_together1} is the triangle inequality. 
\autoref{eqn:put_together2} is \autoref{eqn:incr_bound5}, along with cancelling the term of $||\nabla f||$ top and bottom.
\autoref{eqn:put_together3} are \autoref{eqn:put_together4} are obvious and a standard simplification in the literature on controlling the Lipschitz constant of neural networks. Because exactly computing the smallest Lipschitz constant of a general neural network is NP-complete, a widely-used baseline measure of Lipschitz-smoothness is rather the product of the Lipschitz constants of smaller components of the network, such as single layers (\cite{Scaman2018}). 
\qed 

\section{Loss curvature vs logit curvature}
\label{sec:loss_vs_logit}

We have thusfar discussed how to assure a bound on curvature independent of the input $x$. However, a truly data-independent bound must also not depend on the class label $y$. In this short section we discuss the main consideration that accompanies this: the differences between loss and (pre-softmax) logit curvature. Let the function including the loss for a class $c$ be denoted by $f^c_{\textnormal{loss}}$, and the logit corresponding to a class $c \in [1, C]$ be $f_{\textnormal{logit}}^c$, and let $f_{\textnormal{logit}}^{1,C}$ denote the vector valued function corresponding to all the logits.
Then we have $f^c_{\textnormal{loss}} = \text{lsm}^c \circ f^{1,C}_{\textnormal{logit}}$, where $\text{lsm}^c(\X) = -\log \frac{\exp(\X_c)}{\sum_{i=1}^C \exp(\X_i)} = -\X_c + \log \sum_{i=1}^C \exp(\X_i)$ is the negative log softmax function.

The derivatives of this function are:

\begin{align*}
    \frac{d \text{lsm}^c(\X)}{d \X_c} &= \frac{\exp(\X_c)}{\sum_{i=1}^C \exp(\X_i)} - 1 \\
    \frac{d^2 \text{lsm}^c(\X)}{d \X_c^2} &= \frac{\exp(\X_c)}{\sum_{i=1}^C \exp(\X_i)} \left(1 - \frac{\exp(\X_c)}{\sum_{i=1}^C \exp(\X_i)} \right)  \\
\end{align*}

whose norm is upper bounded by $1$ and $0.25$ respectively. This implies that bounding the logit curvature 
also ensures that the loss curvature is bounded, as the gradients and Hessians of the negative log softmax layer do not explode. In other words, $\C_{f_{\textnormal{loss}}} < \C_{f_{\textnormal{\textnormal{logit}}}}$.

However penalizing an upper bound may not be the most efficient way to penalize the loss Hessian. Consider the following decomposition of the loss Hessian, which can be written as $\grad^2 f(\X) \sim \grad f_{\textnormal{logit}}(\X) \nabla^2_{f_l} \text{LSE}(\X) \grad f_{\textnormal{logit}}(\X)^{\top} + \nabla_{f_l} \text{LSE}(\X) \grad^2 f_{\textnormal{logit}}(\X)$, where $\text{LSE}(\X)$ is the LogSumExp
function. Thus a more efficient way to penalize the loss Hessians is to both use our LCNN penalties, as well as
penalize the gradient norm, which we also find to be true in our experiments.

\section{$\gamma$-Lipschitz Batchnorm}

We present here a Pytorch-style pseudo-code for the $\gamma$-Lipschitz Batchnorm for clarity.

\begin{verbatim}
  bn = torch.nn.BatchNorm2d(in_channels, affine=False)
  log_lipschitz = torch.nn.Parameter(torch.tensor(init_lipschitz).log())
  ...
  # perform spectral normalization for BN in closed form
  bn_spectral_norm = torch.max(1 / (bn.running_var + 1e-5).sqrt()) 
  one_lipschitz_bn = bn(x) / bn_spectral_norm
  
  # multiply normalized BN with a learnable scale 
  scale = torch.min((bn.running_var + 1e-5) ** .5)
  one_lipschitz_part = bn(x) * scale
  x = one_lipschitz_part * torch.minimum(1 / scale, log_lipschitz.exp())
\end{verbatim}
Note that we parameterize $\gamma$ in terms of its log value, to ensure it remains positive during training. We employ the same method for fitting the centered softplus parameter, $\beta$.

\section{Proofs of LCNN Properties}

Here we present proofs for the properties linking LCNNs to gradient smoothness and adversarial robustness.
To this end, we first prove a lemma that is used in both these results, which states that the gradient 
norm can rise exponentially in the worst case. In the text below, $B_{\delta}(\X)$ denotes the 
$\delta$-ball around $\X$. 

\begin{lemma}\label{lemma:gradratio}
The ratio of gradient norms at nearby points rises exponentially with distance between those 
points. For a function $f$ satisfying $\max_{\X' \in B_{\delta}(\X)} \C_f(\X') = \delta_{\C}$, points $\X$ and $\X + \epsilon$ that are $\| \epsilon \|_2 = r$ away, we have 
\begin{align*}
\frac{\| \grad f(\X + \epsilon) \|_2}{\| \grad f(\X) \|_2} \leq \exp(r \delta_{\C})
\end{align*}

\begin{proof}
Let $g(\X) = \log \| \grad f(\X) \|^2_2$. Applying a first order Taylor series with Lagrange 
remainder / mean-value theorem, we have 

\begin{align*}
g(\X + \epsilon) - g(\X) &= \grad g(\X + \xi)^\top \epsilon ~~~~ &(\text{for some } \xi \in B_{\epsilon}(\X))\\
\log \frac{\| \grad f(\X + \epsilon) \|^2_2}{\| \grad f(\X) \|^2_2} &= 2 \frac{\grad f(\X + \xi)\grad^2 f(\X + \xi)^\top \epsilon}{\| \grad f(\X + \xi) \|^2_2} \\
& \leq 2 \frac{\| \grad f(\X + \xi) \|_2 \| \grad^2 f(\X + \xi)\|_2 \| \epsilon \|_2}{\| \grad f(\X + \xi) \|^2_2} ~~~~ &(\text{Cauchy-Schwartz inequality})\\
& \leq 2 r \delta_{\C}
\end{align*}

Note that to apply the Cauchy-Schwartz inequality, we first upper bound the term on the right by
its 2-norm, which for a scalar is simply its absolute value. Taking exponent on both sides of the 
final expression and taking square root, we have the intended result.

\end{proof}

\end{lemma}

\subsection{LCNNs have Robust Gradients}

\begin{proposition}
  Let $\max_{\X' \in B_{\epsilon}(\X)} \C_f(\X') \leq \delta_{\C}$,  then the relative distance between gradients at $\X$ and $\X + \epsilon$ is

\begin{align*}
  \dfrac{\| \grad f(\X + \epsilon) - \grad f(\X) \|^2_2}{\| \grad f(\X) \|^2_2 } 
    \leq r \delta_{\C} \exp(r \delta_{\C}) \sim r \C_f(\X) ~~~~(\text{Quadratic Approximation})
\end{align*}

\end{proposition}

\begin{proof} We begin the proof by invoking the Taylor series approximation of $\grad f(\X +
\epsilon)$ at $\X$, and using the explicit Lagrange form of the Taylor error. This is equivalent
to using a form of the multivariate mean-value theorem. Let $\xi \in B_{\epsilon}(\X)$, then there
exists some $\xi$ such that the following holds

\begin{align*}
  \grad f(\X + \epsilon) - \grad f(\X) 
    & = \grad^2 f(\X+ \xi)^\top \epsilon ~~~~~ &(\text{Taylor Series})\\
  \| \grad f(\X + \epsilon) - \grad f(\X) \|_2  
    &\leq \| \grad^2 f(\X + \xi) \|_2 \| \epsilon \|_2 ~~~~~&(\text{Cauchy- Schwartz inequality})\\ 
  \frac{\| \grad f(\X + \epsilon) - \grad f(\X) \|_2}
  {\| \grad f(\X) \|_2 } 
    &\leq \frac{\C_f(\X + \xi_0) \| \grad f(\X + \xi) \|_2}{\| \grad f(\X) \|_2} r ~~~~~ &(\textit{Divide by gradnorm})
\end{align*}

Plugging in the value of $\| \grad f(\X + \xi_0) \|_2 / \| \grad f(\X) \|_2$ from Lemma \ref{lemma:gradratio}, and further upper bounding $\C_f(\X + \xi) \leq \delta_{\C}$ we have the intended result.

To derive the simplified quadratic approximation, replace the first step in the Taylor series with 
$\grad f(\X + \epsilon) - \grad f(\X)  = \grad^2 f(\X)^\top \epsilon$, i.e., use $\xi = 0$.
\end{proof}

\subsection{Curvature is Necessary for Robustness}

\begin{proposition}
  Let $\max_{\X' \in B_{\epsilon}(\X)} \C_f(\X') \leq \delta_{\C}$, then for two points $\X$ and $\X + \epsilon$,

  \begin{align*}
    \| f(\X + \epsilon) - f(\X) \|_2 
      \leq r \| \grad f \|_2 \left( 1 + 
          \frac{1}{2} r \delta_{\C}  \exp(r \delta_C) \right) 
    \sim r \| \grad f \|_2 \left( 1 + 
          \frac{1}{2} r \C_f(\X)  \right) ~~~~(\text{Quadratic Approximation})
  \end{align*}
\end{proposition}

\begin{proof} Let $\xi \in B_{\epsilon}(\X)$, then there exists some $\xi$ such that
  the following holds
  \begin{align*}
  f(\X + \epsilon) - f(\X) 
    &=\grad f(\X)^\top \epsilon + \frac{1}{2} \epsilon^\top \grad^2 f(\X + \xi) \epsilon ~~ &(\text{Taylor Series})\\
  \| f(\X + \epsilon) - f(\X) \|_2 
  &\leq \| \grad f(\X)^\top \epsilon \|_2 +  \frac{1}{2} \|\epsilon^\top \grad^2 f(\X + \xi) \epsilon \|_2 
   ~~ &(\text{triangle inequality}) \\
  &\leq \| \grad f(\X) \|_2 \| \epsilon \|_2 + \frac{1}{2} \lambda_{\max}(\X + \xi) \|\epsilon \|^2_2 
   ~~ &(\text{Cauchy-Schwartz inequality})\\
  &\leq \| \grad f(\X) \|_2 \| \epsilon \|_2   + \frac{1}{2} \C_f(\X + \xi) 
    \|\epsilon \|^2_2 \| \grad f(\X + \xi) \|_2 
   ~~ &(\text{Defn of } \C_f(\X + \xi)) 
  \end{align*}

Factoring out $\| \grad f(\X) \|_2 \| \epsilon \|_2$ in the RHS, and using Lemma \ref{lemma:gradratio}, 
and further upper bounding $\C_f(\X + \xi) \leq \delta_{\C}$ we have the intended result.

To derive the simplified quadratic approximation, replace the first step in the Taylor series with 
$ f(\X + \epsilon) - f(\X) =\grad f(\X)^\top \epsilon + \frac{1}{2} \epsilon^\top \grad^2 f(\X) \epsilon$, 
i.e., use $\xi = 0$.

  \end{proof}

\section{Experimental Settings}

In this section we elaborate on the hyper-parameter settings used for our tuning our models.
For the standard ResNet-18, we use standard hyper-parameter settings as indicated in the main paper, and we 
do not modify this for the other variants. For LCNNs, we chose regularizing constants as $\lambda_{\beta} = 10^{-4}$ and $\lambda_{\gamma} = 10^{-5}$.  For GradReg, we use $\lambda_{\textnormal{grad}} = 10^{-3}$, and for LCNNs + GradReg, we chose $\lambda_{\beta} = 10^{-4}$, $\lambda_{\gamma} = 10^{-5}$, $\lambda_{\textnormal{grad}} = 10^{-3}$. We performed a coarse grid search 
and chose the largest regularizing constants that did not affect predictive performance.

\section{Additional Experiments}

\subsection{Ablation Experiments} 
In this section, we present ablation studies where we train models with each of the proposed modifications separately, i.e., we train a model with only spectral norm for the convolution layers, $\gamma$-Lipschitz Batchnorm or centered softplus. Our results in Table \ref{table:ablations} show that for the resnet18 architecture considered,

\begin{table}[h]
  \begin{center}
    \caption{Ablation experiments to study effect of individual modifications to LCNN architectures. We find that while either using only centered softplus or $\gamma$-BN suffices in practice to minimize curvature, while spectral norm on the convolutional layers (which is the most expensive modification) may not be necessary.}
    \label{table:ablations}
    
    \begin{tabular}{c | c | c | c | c} 
 \textbf{Model} &    $\E_{\X} \C_f(\X)$     & $\E_{\X} \| \grad^2 f(\X) \|_2$ & $\E_{\X} \| \grad f(\X) \|_2$ & \textbf{Accuracy (\%)} \\
\midrule 
 ConvSpectralNorm only 	& 358.86 	  & 8380.55   & 23.92     & 77.55  \\
$\gamma$-BN only        & 65.78           & 1086.95   & 20.86     & 77.33  \\
c-Softplus only 	& 57.49 	  & 734.05    & 16.99     & 77.31  \\
\midrule
Standard & 270.89  & 6061.96 & 19.66  & 77.42 \\
LCNN & 69.40           & 1143.62   & 22.04     & 77.30  \\
    \end{tabular}
  \end{center}
  \end{table}

(1) performing spectral normalization had no effect of curvature as presumably the batchnorm layers are able to compensate for the lost flexibility, and 

(2) either penalizing the batchnorm alone or the softplus alone performs almost as well as LCNN, or sometimes even better in terms of curvature reduction. Note that the most expensive computational step is the spectral norm for the convolutional layers, indicating that avoiding this in practice may yield speedups in practice.

While in practice for Resnet-18 only $\gamma$-Lipschitz or centered softplus is sufficient for curvature reduction, in theory we must regularize all components to avoid scaling restrictions in one layer being compensated by other layers, as dictated by the upper bound. In particular, this means that the strategy of penalizing only a subset of layers may not generalize to other architectures.

\subsection{Robustness Evaluation on RobustBench / Autoattack}

The attack presented in \autoref{sec:expts} was relatively ``weak'' -- a network could be truly susceptible to adversarial attack, but by only testing against a weak adversary, we could fail to notice. Short of a comprehensive verification (e.g. \cite{Katz2017}), which is known to be computationally intractable at scale, there is no fully satisfactory way to guarantee robustness. However, one common method to develop confidence in a model is to demonstrate robustness in the face of a standard set of nontrivially capable attackers. This is what we do in Table \ref{table:adv_robustness_robustbench}, where we use the Robustbench software library \cite{Croce2021} to evaluate both a white-box (having access to the internal details of the model), and a black-box (using only function evaluations) attacks. 

\begin{table}[h]
  \begin{center}
    \caption{Adversarial accuracy to standard attacks accessed via Robustbench \cite{Croce2021}. ``APGD-t'' refers to the white box targetted auto PGD attack from \cite{Croce2020}, ``Square'' refers to the black-box square attack from \cite{Andriushchenko2019}.}
    \label{table:adv_robustness_robustbench}
    \begin{tabular}{c | c | c | c} 
\textbf{Model}                  	      & \textbf{APGD-t Acc. (\%)}  &  \textbf{Square Acc. (\%)} & \textbf{Clean Acc. (\%)} \\
\midrule  
Standard          	      &22.122	      & 52.874	     & 76.721    \\ 
LCNN              	      &23.709	      & 52.179	     & 76.602    \\ 
GradReg                       &50.294	     & 64.678	     & 76.394    \\ 
LCNN+GradReg 	              &52.477	     & 64.678	     & 76.622   \\ 
\midrule                                                              
CURE 	                      &50.096	      & 63.488	     & 75.928    \\ 
Softplus+wt decay 	      &23.907	      & 53.766	     & 76.622   \\ 
Adv Training                  &55.155	      & 66.861	    & 75.521    \\ 
\midrule                                                                 
CURE + GradReg 	              &60.810	      & 67.357	     & 74.192   \\ 
LCNN + GradReg + Adv Training &59.222	      & 66.861	\    & 75.382    \\ 

    \end{tabular}
  \end{center}
  \end{table}

These experimental results show overall that:

(1) LCNN + GradReg is still on par with adversarial training even against stronger attacks such as APGD-t and Square, as they are with PGD.

(2) Combining LCNN + GradReg with adversarial training (in the last row) further improves robustness at the cost of predictive accuracy.

(3) Combining CURE with GradReg (in the penultimate row) improves robustness at the cost of further deteriorating predictive accuracy.

\subsection{Additional Evaluations on more Architectures / Dataset Combinations}

We present results on the following dataset - architecture pairs: 

(1) In Table \ref{tab:geometry1}, we present results on SVHN dataset and VGG-11 model

(2) In Table \ref{tab:geometry2}, we present results on SVHN dataset and ResNet-18 model

(3) In Table \ref{tab:geometry3}, we present results on CIFAR-100 dataset and VGG-11 model

In all cases, we find that our results are on par with our experiment done on the CIFAR-100 and ResNet-18 setup, which confirms that generality of our approach.

\begin{table}[h]
  \begin{center}
    \caption{Model geometry of VGG-11 models trained on the SVHN test dataset.}
  \label{tab:geometry1}
    \begin{tabular}{c | c | c | c | c} 
     \textbf{Model} & $\E_{\X} \| \grad f(\X) \|_2$ & $\E_{\X} \| \grad^2 f(\X) \|_2$ & $\E_{\X} \C_f(\X)$ & Accuracy (\%) \\ 
     \midrule

     Standard & 2.87 & 158.29 & 54.24 & 96.01   \\ 
     LCNNs & 4.04 & 83.34 & 30.05 & 95.61 \\
     GradReg \cite{Drucker1992} & 1.85 & 57.52 & 33.34 & 96.03  \\ 
     LCNNs + GradReg & 2.02 & 25.54 & 17.06 & 96.23  \\
     \midrule
     Adversarial Training \cite{Madry2017} & 1.25 & 27.64 & 24.23 & 96.37\\ 
    \end{tabular}

  \end{center}
  \end{table}

\begin{table}[h]
  \begin{center}
    \caption{Model geometry of Resnet-18 models trained on the SVHN test dataset.}
  \label{tab:geometry2}
    \begin{tabular}{c | c | c | c | c} 
     \textbf{Model} & $\E_{\X} \| \grad f(\X) \|_2$ & $\E_{\X} \| \grad^2 f(\X) \|_2$ & $\E_{\X} \C_f(\X)$ & Accuracy (\%) \\ 
     \midrule

     Standard & 2.64 & 204.38 & 78.22 & 96.41   \\ 
     LCNNs & 2.91 & 77.78 & 25.36 & 96.35 \\
     GradReg \cite{Drucker1992} & 1.63 & 68.22 & 39.55 & 96.57  \\ 
     LCNNs + GradReg & 1.69 & 31.69 & 15.28 & 96.53  \\
     \midrule
     Adversarial Training \cite{Madry2017} & 1.05 & 22.96 & 24.48 & 96.64\\ 
    \end{tabular}
  \end{center}
  \end{table}

\begin{table}[h]
  \begin{center}
    \caption{Model geometry of VGG-11 models trained on the CIFAR-100 test dataset.}
  \label{tab:geometry3}
    \begin{tabular}{c | c | c | c | c} 
     \textbf{Model} & $\E_{\X} \| \grad f(\X) \|_2$ & $\E_{\X} \| \grad^2 f(\X) \|_2$ & $\E_{\X} \C_f(\X)$ & Accuracy (\%) \\ 
     \midrule

     Standard & 17.07 & 1482.16 & 85.81 & 73.33   \\ 
     LCNNs & 15.88 & 282.06 & 41.14 & 73.76 \\
     GradReg \cite{Drucker1992} & 10.64 & 534.71 & 48.26 & 72.65  \\ 
     LCNNs + GradReg & 9.81 & 105.07 & 24.48 & 73.01  \\
     \midrule
     Adversarial Training \cite{Madry2017} & 6.20 & 166.73 & 27.37 & 71.13\\ 
    \end{tabular}

  \end{center}
  \end{table}

\end{document}